\documentclass{article}
\usepackage{multicol}
\usepackage{algorithm}
\usepackage{algpseudocode}
\usepackage{mathtools}
\usepackage{amsmath}
\usepackage{amssymb}
\usepackage{verbatim}
\usepackage{breqn}
\usepackage{placeins}
\usepackage{comment}
\usepackage{booktabs}
\usepackage{subcaption}
\usepackage{amsthm}

\title{\large \bf Finite-time analysis of Multi-timescale Stochastic Optimization Algorithms}
\author{Kaustubh Kartikey and Shalabh Bhatnagar}
\date{}
\usepackage{natbib} 
\begin{document}
\maketitle
\thispagestyle{empty}
\pagestyle{empty}

\newtheorem{assumption}{Assumption}
\newtheorem{lemma}{Lemma}
\newtheorem{proposition}{Proposition}
\newtheorem{corollary}{Corollary}
\newtheorem{theorem}{Theorem}
\newtheorem{remark}{Remark}

\begin{abstract}

We present a finite-time analysis of two smoothed functional stochastic approximation algorithms for simulation-based optimization. The first is a two time-scale gradient-based method, while the second is a three time-scale Newton-based algorithm that estimates both the gradient and the Hessian of the objective function $J$. Both algorithms involve zeroth order estimates for the gradient/Hessian. Although the asymptotic convergence of these algorithms has been established in prior work, finite-time guarantees of two-timescale stochastic optimization algorithms in zeroth order settings have not been provided previously. For our Newton algorithm, we derive mean-squared error bounds for the Hessian estimator and establish a finite-time bound on $\min\limits_{0 \le m \le T} \mathbb{E}\|\nabla J(\theta(m))\|^2$, showing convergence to first-order stationary points. The analysis explicitly characterizes the interaction between multiple time-scales and the propagation of estimation errors. We further identify step-size choices that balance dominant error terms and achieve near-optimal convergence rates. We also provide corresponding finite-time guarantees for the gradient algorithm under the same framework. The theoretical results are further validated through experiments on the Continuous Mountain Car environment.

\end{abstract}

\section{INTRODUCTION}

Stochastic approximation (SA) methods form the backbone of simulation-based optimization, reinforcement learning, and adaptive control algorithms, where objective functions and gradients are not available in closed form but can be estimated through noisy observations. While gradient-based algorithms such as stochastic gradient descent and simultaneous perturbation stochastic approximation (SPSA) are widely used due to their simplicity and robustness, their convergence can be slow in ill-conditioned problems. This has motivated the development of Newton-based stochastic methods that aim to exploit curvature information to improve convergence speed and stability. In cases where the gradient and/or Hessian of the sample performance objective is not known in closed form, zeroth-order stochastic optimization approaches \cite{c14} help by providing suitable forms of the gradient/Hessian estimates based on noisy performance samples.

Newton-type stochastic approximation algorithms that estimate the Hessian or its inverse using noisy function evaluations have attracted significant attention recently. Notable examples include Newton-based simultaneous perturbation stochastic approximation (SPSA) \cite{c12,c13} and smoothed functional (SF)–based methods such as the Newton smoothed functional (NSF) algorithms \cite{c6}. These methods typically operate on multiple time-scales, with the Hessian estimator evolving on a faster time-scale than the parameter updates as it helps obtain the Hessian inverse post the averaging on the Hessian. Although the asymptotic convergence properties of these previously proposed algorithms are well understood, their analysis of finite-time behavior remains limited or unknown.

Finite-time convergence analysis has been extensively studied for single timescale gradient-based stochastic optimization methods, including nonconvex settings, where bounds on the expected gradient norm or function suboptimality are available \cite{c14}. However, extending such analyses to multi-timescale, Newton-based  algorithms is challenging due to the coupling between parameter updates and the noisy Hessian estimates.
In particular, note that the Hessian is learned on a faster time-scale while the parameter updates depend on its inverse and proceed on a slower scale. 
Thus, bounding the transient error of the Hessian estimator and quantifying its impact on the optimization dynamics requires analysis of multiple interacting stochastic recursions.

This paper addresses this problem by providing a finite-time analysis of a three-timescale Newton smoothed functional (NSF1) algorithm. We focus on bounding both the mean-squared error of the Hessian estimator and the expected norm of the gradient over a $T$-length horizon, thereby establishing convergence to first-order stationary points with explicit rates.

The main contributions of this paper are summarized as follows:
\begin{itemize}
    \item We derive finite-time mean-squared error bounds for the Hessian estimator in the NSF1 algorithm, explicitly accounting for both diagonal and off-diagonal estimation errors.
    \item Using these bounds, we establish a finite-time convergence guarantee on
    $\min\limits_{0\le m\le T}\mathbb{E}\|\nabla J(\theta(m))\|^2$, demonstrating convergence to
    first-order stationary points.
    \item We characterize the interaction between multiple timescale recursions and identify
    near-optimal step-size exponents that balance the dominant error terms so as to minimize the sample complexity 
    and yield near-optimal finite-time rates.
    \item Numerical experiments validate the theoretical decay rates and illustrate the benefits
    of curvature-aware updates compared to purely gradient-based stochastic methods.
\end{itemize}

\section{Related Work}

This work analyzes stochastic approximation algorithms involving smoothed functional gradient/Hessian estimators, for their finite-time behavior.
We briefly review the most relevant literature and state our contributions relative to existing results.

Classical stochastic approximation was introduced in the seminal work of ~\cite{c1} and has since been extensively studied. Finite-time convergence guarantees for first-order stochastic optimization methods, including nonconvex settings, are now well understood. Results include convergence rate analyses for stochastic gradient descent and mirror descent ~\cite{c2,c3}.

These analyses typically assume access to unbiased (or biased but controlled) gradient estimates and primarily conducted for algorithms that operate on a single time-scale. In contrast, the algorithms studied in this paper rely solely on noisy function evaluations and involve multi-timescale stochastic iterates. As a result, existing finite-time bounds for first-order methods are not directly applicable to the setting considered here.

Smoothed functional (SF) and simultaneous perturbation stochastic approximation (SPSA) methods provide gradient and Hessian estimates using random perturbations and noisy function evaluations. Early developments include SPSA~\cite{c4} and SF-based methods for simulation optimization~\cite{c5}. Second-order extensions, including Newton SPSA and Newton smoothed functional (NSF) algorithms, were developed in~\cite{c12, c13, c6}.

Most existing analyses of SF and SPSA methods focus on providing asymptotic convergence guarantees, establishing almost sure convergence or convergence in probability under standard stochastic approximation assumptions.
While these results provide important theoretical guarantees, they do not yield explicit finite-time error bounds for the gradient or Hessian estimates. In particular, finite-time mean-squared error bounds for Hessian estimators in NSF-type algorithms have remained largely unexplored.

Second-order stochastic approximation methods are known to improve conditioning and convergence behavior compared to first-order methods, especially in ill-conditioned problems. Classical treatments of second-order stochastic approximation rely on ODE-based or weak convergence arguments~\cite{c7,c8, c9}. These analyses are inherently asymptotic and do not quantify transient behavior.

Our work complements this literature by providing the first explicit finite-time bounds for Newton-type smoothed functional algorithms. Unlike prior asymptotic results, we explicitly characterize the evolution of the Hessian estimation error and quantify its impact on the optimization dynamics over a horizon of length $T$.

Multi-timescale stochastic approximation has been widely studied in the context of adaptive control and reinforcement learning (RL)~\cite{c11, c10, c9}. These works establish asymptotic convergence using ODE and differential inclusion
techniques and have been successfully applied to actor-critic and policy-gradient algorithms.

Finite-time analysis of multi-timescale algorithms even in the context of RL, however, is significantly more challenging and remains relatively limited. 
Existing results largely focus on first-order methods (see \cite{c14}) or special linear settings. To the best of our knowledge, finite-time convergence guarantees for multi-timescale, second-order smoothed functional algorithms have not been previously established.

The results in this paper can be viewed as a finite-time counterpart to existing asymptotic analyses of NSF algorithms. Compared to first-order smoothed functional methods such as GSF1, our analysis demonstrates how incorporating curvature information improves convergence behavior
under comparable oracle and noise assumptions. While the resulting rates are generally slower than those obtained for stochastic gradient methods with exact (noisy) gradient access, they are achieved under substantially
weaker information models, relying only on noisy function evaluations and random perturbations.

In summary, this work fills a major gap that previously existed by providing the first finite-time analysis of Newton-based multi-timescale stochastic optimization algorithms. 

\section{NSF1 and GSF1 Algorithms}

In this section, we describe the Newton Smoothed Functional (NSF1) algorithm, a three–time–scale stochastic approximation method that simultaneously estimates the Hessian and gradient of the objective function using noisy function measurements  and performs curvature-aware parameter updates. The algorithm is designed to ensure that curvature information is tracked rapidly, while the parameters evolve on a slower time-scale to guarantee stability.

\subsection{Problem Setup}

We consider the unconstrained stochastic optimization problem
\[
\min_{\theta \in \mathbb{R}^d} J(\theta),
\]
where ${\displaystyle J(\theta)=\lim_{N\rightarrow\infty} \frac{1}{N} \mathbb{E} [\sum_{m=0}^{N-1} h(\theta,X_m)]= \sum_{s\in S}d_\theta(s) h(\theta,s)}$ is a long-run average over single-stage costs $h(\theta,X_m)$ that depend on a parameter $\theta$ and where given any parameter $\theta$, $\{X_m\}$ is an ergodic Markov chain with transition probabilities $p_\theta(s,s')$, $s,s'\in S$ with stationary distribution $d_\theta(s),s\in S$. For simplicity, we assume $S$ to be a finite set. 
The objective function $J(\theta)$ and its derivatives are not directly observable;
instead, noisy samples $h(\theta,X_m)$, $m\geq 1$, are available.

\subsection{Smoothed Functional Gradient and Hessian Estimation}

Let $\{\eta(m)\}_{m\ge0}$ be an i.i.d.\ sequence of $d$-dimensional random vectors with each component being distributed $N(0,1)$.
For a smoothing parameter $\beta>0$, the smoothed functional gradient estimator is defined as 
\[
z(m)
= \frac{\eta(m)}{\beta}\,h(\theta(m)+\beta\eta(m),X_m).
\]

The smoothed functional Hessian estimator is defined element-wise as 
\begin{equation}
    \hat{H}_{i,i}(m)
= \frac{\eta_i^2(m)-1}{\beta^2}\,h(\theta(m)+\beta\eta(m),X_m),
\end{equation}
$i=1,\ldots,d$ and 
\begin{equation}
\hat{H}_{i,j}(m)
= \frac{\eta_i(m)\eta_j(m)}{\beta^2}\,h(\theta(m)+\beta\eta(m),X_m),
\quad i\neq j.
\end{equation}
As shown in \cite{c6}, these estimators provide unbiased estimates of the gradient and Hessian of the
Gaussian-smoothed objective in the limit as $\beta\rightarrow 0$.

\subsection{Recursive Hessian Tracking (Fastest Time-Scale)}

The Hessian estimate is updated recursively as
\begin{equation}
    H(m+1)
= H(m) + b(m)\big(\hat{H}(m)-H(m)\big),
\end{equation}
where $\hat{H}(m)=[[\hat{H}_{i,j}(m)]]_{i,j=1}^{d}$, 
$0<b(m)\rightarrow 0$ slowest amongst all the step-size schedules, making the recursion governed by $b(m)$ the fastest timescale recursion. 
This update ensures that the Hessian estimator rapidly tracks
$\nabla^2 J(\theta(m))$ (see \cite{c6} for details).

\subsection{Gradient Tracking (Intermediate Time-Scale)}

In addition to Hessian tracking, the NSF1 algorithm maintains a recursive estimate
of the gradient on an intermediate time-scale.
Let $Z(m)$ denote the gradient estimate.
Using smoothed functional perturbations, an instantaneous gradient observation is given by

\begin{equation}
\label{grad1}
    \hat{g}(m)
= \frac{\eta(m)}{\beta}\,h(\theta(m)+\beta\eta(m),X_m).
\end{equation}
The gradient estimate is updated recursively as
\begin{equation}
\label{grad2}
Z(m+1)
= Z(m) + c(m)\big(\hat{g}(m) - Z(m)\big),
\end{equation}
where $c(m)>0$ is a diminishing step-size satisfying
\[
a(m)=o( c(m)) \text{ and } c(m)=o (b(m)).
\]
This recursion allows the gradient estimator to track
$\nabla J(\theta(m))$ on a time-scale slower than the Hessian estimation
but faster than the parameter updates. 
Under standard smoothness assumptions, in the limit as $\beta\rightarrow 0$, $Z(m)$ converges to the true gradient
$\nabla J(\theta(m))$ sufficiently fast so that its tracking error does not
dominate the parameter dynamics.

\subsection{Parameter Update (Slowest Time-Scale)}

The parameter update is given by
\[
\theta(m+1)
= \Gamma(\theta(m) - a(m)\,M(m)\,Z(m)),\]
%
where $M(m) = (P_{\text{PD}}(H(m)))^{-1}$ with 
$P_{\text{PD}}$ being an operator that projects any matrix to the space of positive definite matrices. This may correspond to the Choleski factorization procedure. It could correspond to 
$M(m) = (H(m)+ \lambda I)^{-1}$ for a small enough $\lambda>0$.
This projection is needed to ensure the updates in the descent direction guided by the Newton step.
Also, in the above, $\Gamma:\mathbb{R}^d\rightarrow C\subset\mathbb{R}^d$ is a projection operator to a prespecified compact and convex set $C$. 


 \begin{algorithm}[H]
\caption{Newton Smoothed Functional (NSF1) Algorithm}
\label{alg:nsf1}
\begin{algorithmic}[1]
\Require Initial parameter $\theta(0)$, initial Hessian estimate $H(0)$, smoothing parameter $\beta > 0$, total iterations $T$
\State Step-sizes: $a(t) = (1+t)^{-\sigma}$, $b(t) = (1+t)^{-\alpha}$, $c(t) = (1+t)^{-\nu}$, with $0 < \nu < \alpha < \sigma < 1$.
\For{$t = 0$ to $T-1$}
    \State Sample perturbations $\eta_i(t) \sim \mathcal{N}(0, 1)$, $i=1,\ldots,d$ and $\eta(t) = (\eta_i(t),i=1, 2,\ldots, d)^T$.
    \State Generate $X(t)$ using the perturbed parameter $(\theta(t) + \eta(t))$.
    
    \State \textbf{Estimate the smoothed Hessian (using a moving average):}
    for  $i,j,k = 1, 2, \ldots, d$, 
    \begin{equation*}
        \begin{aligned}
            H_{i,i}(t+1) &= H_{i,i}(t) + b(t) \Big( \frac{\eta_i^2(t) - 1}{\beta^2} h(X(t)) - \\ H_{i, i}(t) \Big)
        \end{aligned}
    \end{equation*}
For $k<j$,
     \begin{equation*}
     \begin{aligned}
          H_{j,k}(t+1) &= H_{j,k}(t) + b(t) \Big( \frac{\eta_j(t) \eta_k(t) - 1}{\beta^2}  h(X(t)) - 
          \\ H_{j, k}(t) \Big)
    \end{aligned}
     \end{equation*}          
   For $j<k$, set 
    $H_{j,k}(t+ 1) = H_{k,j}(t+1)$.
    
    \State Project Hessian to ensure positive definiteness:
        \[H(t+1) \leftarrow P_{\text{PD}}\big(H(t+1)\big)\]

    \State \textbf{Estimate the smoothed gradient:}
    \begin{equation*}
         Z(t+1) = Z(t) + c(t)\left(\frac{\eta(t)}{\beta}h(X(t)) - Z(t)\right)
    \end{equation*}
    
    \State Update parameter using a damped Newton step:
    
        $\theta(t+1) = \Gamma(\theta(t) - a(t) \, H^{-1}(t+1) \, Z(t+1))$
    
    where $\Gamma$ is the projection operator.
\EndFor
\State \Return Final parameter $\theta(T)$
\end{algorithmic}
\end{algorithm}

\subsection{The GSF1 Algorithm}

In addition to NSF1, we also analyze the Gradient Smoothed Functional (GSF1)
algorithm, which is a zeroth-order two-timescale stochastic gradient scheme. 
GSF1 employs smoothed functional perturbations to estimate the gradient of the
objective function and updates the parameters using this estimate alone. The gradient estimation $\nabla J(\theta)$ is done using (\ref{grad1})-(\ref{grad2}). The parameter update however is now the following:
\[
\theta(m+1)=\Gamma(\theta(m)-a(m)\,Z(m)).
\]

The finite-time bounds that we develop apply directly
to GSF1 by specializing the analysis to the case without Hessian estimation.
This allows us to provide finite-time convergence guarantees for both first-
and second-order smoothed functional methods under a unified set of assumptions.

\section{Assumptions}
\label{sec:assumptions}

In this section, we state the assumptions used throughout the finite-time analysis.
All assumptions are standard in stochastic approximation and smoothed functional
optimization and are sufficient to establish the results for both NSF1 and GSF1. 
Recall that the long-run average cost is
\begin{equation}
J(\theta) = \lim_{n \to \infty} \frac{1}{n} \mathbb{E}_\theta \left[ \sum_{t=0}^{n-1} h(\theta, X_t) \right] = \pi_\theta^T h_\theta
\end{equation}
where $\pi_\theta$ is the stationary distribution and $h_\theta = [h(\theta, 1), \dots, h(\theta, N)]^T$.

\begin{assumption}
\label{a1}
For all $\theta \in C$, the chain is irreducible and aperiodic. The fundamental matrix $Z_\theta = (I - P_\theta + \mathbf{1}\pi_\theta^T)^{-1}$ is well-defined and satisfies $\sup_{\theta \in C} \|Z_\theta\| < \infty$.
\end{assumption}

\begin{assumption}
\label{a2}
The mappings $\theta \mapsto P(\theta)$ and $\theta \mapsto h_\theta$ are $C^3$ on $C$. Hence, $P_\theta$, $h_\theta$, and their first, second and third derivatives are uniformly bounded on $C$.
\end{assumption}

The following lemma provides smoothness and regularity conditions on the function $J$ and builds on the results in  \cite{c15}.

\begin{lemma}
\label{regularity-lemma}
Under Assumptions \ref{a1} and \ref{a2}, $\nabla J(\theta)$ and $\nabla^2 J(\theta)$ are both Lipschitz continuous on $C$.
\end{lemma}

\begin{proof}
We show it for the case of $\nabla J(\theta)$ as similar arguments extend to the case of $\nabla^2J(\theta)$. 
Differentiating the identity $\pi_\theta^T (I - P_\theta) = 0$ with respect to $\theta_k$ yields:
\begin{equation}
\nabla_{\theta} \pi_\theta^T (I - P_\theta) = \pi_\theta^T \nabla_{\theta} P_\theta
\end{equation}
Post-multiplying by $Z_\theta$ and noting that $\nabla_{\theta} \pi_\theta^T \mathbf{1} = 0$ (where $\mathbf{1}$ is the vector of all 1's) gives 
\begin{equation}
\nabla_{\theta} \pi_\theta^T = \pi_\theta^T (\nabla_{\theta} P_\theta) Z_\theta. 
\end{equation}
Then, 
\begin{equation}
\nabla J(\theta) = \nabla \pi_\theta^T h_\theta + \pi_\theta^T \nabla h_\theta = \pi_\theta^T (\nabla P_\theta) Z_\theta h_\theta + \pi_\theta^T \nabla h_\theta
\end{equation}

By the mean-value theorem, $\nabla J(\theta)$ is Lipschitz continuous if $\nabla^2 J(\theta)$) is uniformly bounded on $C$. The components of $\nabla^2 J(\theta)$ involve terms of the form:
$\pi_\theta$, $Z_\theta$, $\nabla P_\theta$, $\nabla^2 P_\theta$, $h_\theta$, $\nabla h_\theta$, $\nabla^2 h_\theta$.
By Assumptions \ref{a1}-\ref{a2}, each of the above terms is continuous and bounded on $C$. Since $Z_\theta$ is an inverse of a matrix whose determinant is non-zero (due to ergodicity), it is also $C^1$ in $\theta$. Also, 
as $\nabla^2 J(\theta)$ is a continuous function on the compact set $C$, there exists a constant $L < \infty$ such that $\sup_{\theta \in C} \|\nabla^2 J(\theta)\| \leq L$. By the mean-value theorem, we then have for any $\theta_1,\theta_2\in C$, 
\begin{equation}
\|\nabla J(\theta_1) - \nabla J(\theta_2)\| \leq L \|\theta_1 - \theta_2\|. 
\end{equation}
The proof for Lipschitz continuity of $\nabla^2J(\theta)$ follows similarly. 
\end{proof}


Lemma~\ref{regularity-lemma} ensures that
the gradient and Hessian vary smoothly along the parameter trajectory and are
used to control drift terms in the finite-time analysis.




\begin{lemma}
\label{lem2}
Under Assumptions~\ref{a1}-\ref{a2},  ${\displaystyle 
   \sup_{\theta \in C}\mathbb{E}[h(\theta,X)^4]}$ $< \infty$. 
\end{lemma}
\begin{proof}
Note that 
\[
\mathbb{E}[h(\theta,X)^4] = \sum_{i=1}^{N} d_\theta(i)h(\theta,i)^4.
\]
The claim follows since $h(\theta,i)$ is continuous on $C$ by Assumption~\ref{a2} and $C$ is compact.   
\end{proof}
%
%
%
%
\begin{remark}
\label{rem1}
It follows from Lemma~\ref{lem2} that
the gradient and Hessian estimators have uniformly bounded second moments:
\begin{equation*}
    \sup_m \mathbb{E}\|\widehat g(m)\|^2 < \infty,
\qquad
\sup_m \mathbb{E}\|H(m)\|_F^2 < \infty.
\end{equation*}
These moment bounds are used to bound variance and cross terms arising from stochastic perturbations. Thus, boundedness of moments holds along the algorithm trajectory due to projection.
\end{remark}


\begin{assumption}
\label{a3}
The step-sizes $a(m),b(m),c(m),m\geq 0$ are selected as 
$a(m)=(1+m)^{-\sigma},\quad
c(m)=(1+m)^{-\alpha},\quad
b(m)=(1+m)^{-\nu},
$
with exponents satisfying
$
0<\nu<\alpha<\sigma<1.
$
\end{assumption}
As a consequence of the above, note that $a(m)=o(c(m))$ and $c(m)=o(b(m))$, thus also $a(m)=o(b(m))$.

We let the operator $P_{\text{PD}}$ acting on any matrix $H$ be specified via $P_{\text{PD}}(H) = (H+\lambda I)$, for some $\lambda>0$ s.t. $(H+\lambda I)$ becomes positive definite.
\begin{assumption}
\label{a4}
There exists $\lambda>0$ such that the regularized Hessian approximation
$M(m)=(H(m)+\lambda I)^{-1}$
is uniformly positive definite for all $m$.
\end{assumption}
Assumption~\ref{a4} ensures stability of the Newton-type updates and is
standard in second order stochastic approximation.



\section{Finite-time analysis of NSF1 and GSF1}
We outline the main ingredients used to prove Theorems~\ref{thm:hessian_mse} and~\ref{thm:gradient_bound}. The full proofs of these results are deferred to the appendix.
We first provide a bound on the mean-squared error associated with the Hessian estimate: $\frac{1}{t + \tau - 1}\sum\limits_{m=\tau}^t E[\|H(m) - H^*(m)\|^2_2]$, where $H^*(m)$ is the Hessian of the objective function at $\theta(m)$ and then use the result obtained to calculate a bound on the term $\min\limits_{0\leq m\leq T} E[\|\nabla J(\theta(m))\|^2]$.

\subsection{Finite-time analysis of NSF1}
\subsection*{Mean-Squared Error Decomposition of the Hessian Estimate}

Let $H(m)\in\mathbb{R}^{d\times d}$ denote the Hessian estimate at iteration $m$, 
and let $H^*(m)=\nabla^2 J(\theta(m))$ denote the true (or smoothed) Hessian evaluated at the current parameter iterate $\theta(m)$. 
The mean-squared error (MSE) of the Hessian estimate is obtained as
\begin{equation*}
    \mathbb{E}\big[\|H(m)-H^*(m)\|_F^2\big]
= \sum_{i=1}^{d}\sum_{j=1}^{d}\mathbb{E}\big[(H_{i,j}(m)-H^*_{i,j}(m))^2\big].
\end{equation*}

\paragraph{Step 1: Splitting into diagonal and off-diagonal components}
We decompose the total error into diagonal and off-diagonal error components:
\[
\mathbb{E}\big[\|H(m)-H^*(m)\|_F^2\big] = 
\underbrace{\sum_{i=1}^{d}\mathbb{E}\big[(H_{i,i}(m)-H^*_{i,i}(m))^2\big]}_{\text{diagonal error}}\]
\[+
\\\underbrace{\sum_{\substack{i,j=1\\ i\neq j}}^{d}\mathbb{E}\big[(H_{i,j}(m)-H^*_{i,j}(m))^2\big]}_{\text{off-diagonal error}}.\]

We derive the finite-time bound explicitly for the diagonal elements, as analogous arguments apply to the off-diagonal terms.

\paragraph{Step 2: Recursion for the diagonal element error}
Define the diagonal error term
$\varepsilon_{i,i}(m) := H_{i,i}(m) - H^*_{i,i}(m),$
and consider its stochastic recursion. 
Ignoring the projection operator for now, the diagonal element terms in the Hessian update in the NSF1 scheme get updated as
\begin{equation*}
H_{i,i}(m+1) 
= (1-b(m)) H_{i,i}(m)
+ b(m)\!
\frac{\eta_i(m)^2 - 1}{\beta^2}\,h(X_m),
\end{equation*}
where $\{\eta_i(m)\}$ are i.i.d.\ $N(0,1)$-distributed perturbations, $\beta>0$ is the perturbation scale, 
and $b(m)$ is the Hessian step-size.
Subtracting the true Hessian term $H^*_{i,i}(m+1)$ and rearranging yields the error recursion:
\begin{align*}
\varepsilon_{i,i}(m+1)&=\varepsilon_{i,i}(m) -b(m)\,\varepsilon_{i,i}(m)\\
&+b(m)\left(\frac{\eta_i(m)^2 - 1}{\beta^2}\,h(X_m) - H^*_{i,i}(m)\right)\\
&
+ \big(H^*_{i,i}(m) - H^*_{i,i}(m+1)\big).
\end{align*}

Squaring both sides and taking expectations produces the one-step MSE recursion for the diagonal element:
\\
\\
$
\mathbb{E}\big[\varepsilon_{i,i}(m+1)^2\big]
= 
\mathbb{E}\big[\varepsilon_{i,i}(m)^2\big]
- 2b(m)\,\mathbb{E}[\varepsilon_{i,i}(m)^2]
+ \\ \text{additional cross terms}.
$

\paragraph{Step 3: Expansion into five components}
The additional terms in the expansion above can be systematically grouped into five categories, corresponding to constants $C_1$--$C_5$ used in the finite-time bound.  
Each term captures a specific stochastic or deterministic contribution to the growth or decay of the mean-squared error:
\begin{align*}
 & \sum\limits_{m=\tau}^tE([\epsilon_{i,i}(m)]^2) \leq \sum\limits_{m=\tau}^tE[\frac{1}{2b(m)}([\epsilon_{i,i}(m)]^2 - \\&[\epsilon_{i,i}(m+1)]^2) ]
+ \sum\limits_{m=\tau}^tE[\epsilon_{i,i}(m) [\frac{\eta_i^2(m) - 1}{\beta^2} h(x(m))  - \\& H_{i,i}^*(m) ]] + \sum\limits_{m=\tau}^tE[\frac{1}{b(m)} \epsilon_{i,i}(m)(H_{i,i}^*(m) - H_{i,i}^*(m+1))]\\
 &+ \sum\limits_{m=\tau}^tE[2 b(m) [\frac{\eta_i^2(m) - 1}{\beta^2} h(x(m)) - H_{i, i}(m)]^2]\\
 &+  \sum\limits_{m=\tau}^tE[\frac{1}{b(m)} (H_{i,i}^*(m)- H_{i,i}^*(m+1) )^2].
 \end{align*}
 
The five constants correspond to the following sources:
\begin{itemize}
    \item \textbf{$C_1$:} Uniform bound on the squared estimation error 
    $(H_{i,i}(m)-H^*_{i,i}(m))^2$, 
    obtained from the projection of $H(m)$ and boundedness of $H^*(m)$.
    \item \textbf{$C_2$:} Cross-term involving $(\eta_i^2-1)h(X_m)$, 
    arising from the stochastic perturbation noise. 
    Bounded by Cauchy--Schwarz using finite moments of $\eta_i$ and $h(X_m)$.
    \item \textbf{$C_3$:} Coupling term capturing the variation of the true Hessian 
    $H^*(m+1)-H^*(m)$; 
    controlled by the Lipschitz continuity of the Hessian and the step-size relation $a(m)/b(m)$.
    \item \textbf{$C_4$:} Martingale variance term 
    due to the zero-mean stochastic component in the update, whose contribution can be bounded due to zero-mean perturbations and the boundedness of the Hessian estimates.
    \item \textbf{$C_5$:} Squared deterministic drift term associated with 
    $\|H^*(m+1)-H^*(m)\|_F^2$; 
    bounded using the Lipschitz property of $H^*$ and the controlled parameter increments.
\end{itemize}

\paragraph{Step 4: Aggregation and Total Hessian Error}
Summing the bounds across all diagonal elements 
and applying analogous (simpler) arguments to the off-diagonal terms yields

\begin{theorem}[{\bf Finite-Time Hessian Estimation Error}]
\label{thm:hessian_mse}
Under Assumptions~\ref{a1}--\ref{a4}, consider the Hessian estimate $H(m)$ generated by the NSF1 algorithm and let $H^*(m) = \nabla^2 J(\theta(m))$. Then, for all $t \ge 1$, the mean-squared error satisfies
\[
\mathbb{E}\big[\|H(t) - H^*(t)\|_F^2\big]
\leq\mathcal{O}\big(t^{\nu}\big)
\;+\;
\mathcal{O}\big(t^{1-\nu}\big)
\;+\;
\mathcal{O}\big(t^{\,1 - 2(\sigma - \nu)}\big),
\]
where the step-size exponents satisfy $0 < \nu < \alpha < \sigma < 1$.

In particular, the Hessian estimation error is governed by the interaction between the Hessian update rate and the drift induced by the parameter updates.
\end{theorem}

\subsection*{Bound on the Expected Norm of the Gradient}

This section establishes a finite-time upper bound on the expected squared norm of the gradient, thereby quantifying the convergence rate of the proposed algorithm. The derivation follows from the smoothness of $J(\theta)$, the parameter update dynamics, and the finite-time error bounds on the Hessian and gradient estimates.

\paragraph{Parameter update and smoothness inequality}
Recall that the parameter update (without the projection) is 
\[
\theta(m+1) = \theta(m) - a(m)\,M(m)\,Z(m),
\]
Using the $L$-smoothness of $J(\theta)$ (from Lemma~\ref{regularity-lemma}), we have
\begin{align*}
J(\theta(m+1)) 
&\le J(\theta(m)) + \langle\nabla J(\theta(m)), \theta(m+1)-\theta(m)\rangle\\
& + \frac{L}{2}\|\theta(m+1)-\theta(m)\|^2\\
&= J(\theta(m))
- a(m)\langle\nabla J(\theta(m)), M(m)Z(m)\rangle\\
& + \frac{L}{2}a(m)^2\|M(m)z(m)\|^2.
\end{align*}

\paragraph{Error decomposition}
Note that  
\begin{align*}
        M(m)z(m) & = [M(m)-M^*(m)]Z(m)\\ &+  M^*(m)[Z(m)-Z^*(m)] + M^*(m)Z^*(m),
\end{align*}
where $M^*(m) = (\nabla^2J(\theta(m)) + \lambda I)^{-1}$, $Z^*(m)=\nabla J(\theta(m))$, respectively. Substituting and rearranging gives
{\small 
\begin{align*}
& J(\theta(m+1)) \le J(\theta(m))
- a(m)\langle\nabla J(\theta(m)), [M(m)-
 M^*(m)]\\& Z(m)\rangle 
- a(m )\langle\nabla J(\theta(m)), M^*(m)[Z(m)-Z^*(m)]\rangle-\\
&a(m)\langle\nabla J(\theta(m)), M^*(m)\nabla J(\theta(m))\rangle
+ \frac{L}{2}a(m)^2\|M(m)Z(m)\|^2.
\end{align*}
}

\paragraph{Bounding the gradient term}
Since $M^*(m)$ is positive definite with the smallest eigenvalue $\lambda_c>0$, 
$x^\top M^*(m)x \ge \lambda_c \|x\|^2$ for all $x$.
Therefore,
\begin{align*}
& a(m)\lambda_c\|\nabla J(\theta(m))\|^2
\le (J(\theta(m)) - J(\theta(m+1))) \\
&- a(m)\langle\nabla J(\theta(m)), [M(m)-M^*(m)]Z(m)\rangle \\
&- a(m)\langle\nabla J(\theta(m)), M^*(m)[Z(m)-Z^*(m)]\rangle\\
&+ \frac{L}{2}a(m)^2\|M(m)Z(m)\|^2.
\end{align*}

Dividing by $a(m)\lambda_c$ and taking expectations yields
\begin{align*}
& \mathbb{E}\|\nabla J(\theta(m))\|^2 \le \frac{1}{a(m)\lambda_c}\mathbb{E}[J(\theta(m)) - J(\theta(m+1))]\\ 
&- \frac{1}{\lambda_c}\mathbb{E}\langle\nabla J(\theta(m)), [M(m)-M^*(m)]Z(m)\rangle\\ 
& - \frac{1}{\lambda_c}\mathbb{E}\langle\nabla J(\theta(m)), M^*(m)[Z(m)-Z^*(m)]\rangle\\
&+ \frac{L a(m)}{2\lambda_c}\mathbb{E}\|M(m)Z(m)\|^2.
\end{align*}

\paragraph{Summation and term-wise bounding}
Summing over $m=\tau$ to $t$ produces
{\small
\begin{align*}
& \sum_{m=\tau}^{t}\mathbb{E}\|\nabla J(\theta(m))\|^2
\le \sum_{m=\tau}^{t}\Bigg[
\frac{\mathbb{E}[J(\theta(m)) - J(\theta(m+1))]}{a(m)\lambda_c}
\\ & + \frac{1}{\lambda_c}\mathbb{E}\langle\nabla J(\theta(m)), -[M(m)-M^*(m)]Z(m)\rangle \\ &+ \frac{1}{\lambda_c}\mathbb{E}\langle\nabla J(\theta(m)), -M^*(m)[Z(m)-Z^*(m)]\rangle\\
&+ \frac{L a(m)}{2\lambda_c}\mathbb{E}\|M(m)Z(m)\|^2
\Bigg].
\end{align*}
}
Each of the four summation terms is bounded individually:

\begin{itemize}
\item[(1)] \textbf{Objective decrease:} 
By telescoping and boundedness of $J$, 
\begin{align*}
\sum_{m=\tau}^{t}\frac{\mathbb{E}[J(\theta(m)) - J(\theta(m+1))]}{a(m)\lambda_c}
& \le \frac{C_1}{a(t)}.
\end{align*}
\item[(2)] \textbf{Matrix approximation error:} 
Using Cauchy--Schwarz inequality and bounds $\|z(m)\|\le C_G$,
{\small
\begin{align*}
&\sum_{m=\tau}^{t}\!\frac{1}{\lambda_c}\mathbb{E}\langle\nabla J(\theta(m)), -[M(m)-M^*(m)]Z(m)\rangle \le \\
& C_2\sqrt{\sum_{m=\tau}^{t}\mathbb{E}\|\nabla J(\theta(m))\|^2}
\sqrt{\sum_{m=\tau}^{t}\mathbb{E}\|M(m)-M^*(m)\|_F^2}.
\end{align*}
}
\item[(3)] \textbf{Gradient estimation error:} 
Since $\|M^*(m)\|_F\le C_H$,
{\small 
\begin{align*}
& \sum_{m=\tau}^{t}\ \frac{1}{\lambda_c}\mathbb{E}\langle\nabla J(\theta(m)),-M^*(m)[Z(m)-Z^*(m)]\rangle \le \\
& C_3\sqrt{\sum_{m=\tau}^{t}
\mathbb{E}\|\nabla J(\theta(m))\|^2}
\sqrt{\sum_{m=\tau}^{t}\mathbb{E}\|Z(m)-Z^*(m)\|^2}.
\end{align*}}
The values of the constants $C_1 - C_4$ have been stated in proofs in the appendix.

\item[(4)] \textbf{ Higher-order step-size term:}
\begin{align*}
\sum_{m=\tau}^{t}\!\frac{L a(m)}{2\lambda_c}\mathbb{E}\|M(m)z(m)\|^2
&\le C_4\sum_{m=\tau}^{t} a(m).
\end{align*}
\end{itemize}

\paragraph{Aggregated bound}
Combining all the terms, we obtain
{\small 
\begin{align*}
& \sum_{m=\tau}^{t}\mathbb{E}\|\nabla J(\theta(m))\|^2
\le \frac{C_1}{a(t)}
\\ 
& + C_2\sqrt{\sum_{m=\tau}^{t}\mathbb{E}\|\nabla J(\theta(m))\|^2}
 \sqrt{\sum_{m=\tau}^{t}\mathbb{E}\|M(m)-M^*(m)\|_F^2} \\
& + C_3\sqrt{\sum_{m=\tau}^{t}\mathbb{E}\|\nabla J(\theta(m))\|^2}
 \sqrt{\sum_{m=\tau}^{t}\mathbb{E}\|Z(m)-Z^*(m)\|^2}\\
&+ C_4\sum_{m=\tau}^{t} a(m).
\end{align*}
}

\paragraph{Substituting step-size schedules}
Now substitute $a(m) = (1+m)^{-\sigma}$ and use the finite-time bounds 
$\mathbb{E}\|M(m)-M^*(m)\|^2 = O(m^{\nu-1})$ and 
$\mathbb{E}\|z(m)-z^*(m)\|^2 = O(m^{\alpha-1})$ 
with $0<\nu<\alpha<\sigma<1$. 

\paragraph{Final bound}
Since $\min\limits_{0\le m\le T}\mathbb{E}\|\nabla J(\theta(m))\|^2 
\le \frac{1}{1+T}\sum_{m=0}^{T}\mathbb{E}\|\nabla J(\theta(m))\|^2$, we obtain the following result on finite-time convergence to first-order stationary points(FOSP).

\begin{theorem}[{Finite-Time Convergence to FOSP}]
\label{thm:gradient_bound}
Under Assumptions 1–4, the iterates of the NSF1 algorithm, 
for any horizon $T \ge 1$, satisfy
\begin{align*}
\min_{0 \le m \le T} \mathbb{E}\|\nabla J(\theta(m))\|^2
&\leq \mathcal{O}\big(T^{\sigma - 1}\big)
+
\mathcal{O}\big(T^{-\nu}\big)\\
&+
\mathcal{O}\big(T^{-2(\sigma - \alpha)}\big).
\end{align*}

The algorithm converges to a FOSP at a sublinear rate determined by the step-size exponents and the propagation of gradient and Hessian estimation errors.
\end{theorem}
\paragraph{Interpretation}
\begin{itemize}
    \item The term $O(T^{\sigma-1})$ represents the dominant convergence rate induced by the parameter step-size $a(m)$.
    \item The terms $O(T^{-\nu})$ and $O(T^{-2(\sigma-\alpha)})$ capture the finite-time propagation of estimation errors from the matrix ($M(m)$) and gradient ($Z(m)$) updates.
    \item Since $0<\nu<\alpha<\sigma<1$, all terms vanish asymptotically, establishing sublinear convergence in expectation.
\end{itemize}

Thus, under Assumptions~\ref{a1}--\ref{a4} and the three-timescale step-size design, the algorithm achieves a finite-time bound on the expected squared gradient norm that ensures the convergence of the sequence $\{\theta(m)\}$ to a stationary point.

Solving for the optimal rate of convergence, we get $\nu \rightarrow 0.4, \alpha \rightarrow \nu,$ and $\sigma \rightarrow 0.6$ gives the optimal rate of the order $O(-0.4+\delta)$, where $\delta \rightarrow 0$. These values are obtained by balancing the dominant finite-time error terms.

\paragraph{Conclusion}
All constants $C_1$--$C_5$ are thus finite and depend only on the fixed quantities
$(B_Z, L, L_H, C_\theta, H_h, \beta, d, c)$.
No additional requirements beyond Assumptions \ref{a1}--\ref{a4} are needed. 
The bounds follow directly from standard norm inequalities, 
Lipschitz smoothness of the objective, 
and bounded-moment conditions on the stochastic perturbations.

\subsection{Finite time analysis of GSF1 }
\begin{theorem}[{ Finite-Time Gradient Estimation Error}]
\label{thm:gradient_mse}
Under Assumptions~\ref{a1}--\ref{a4}, consider the Hessian estimate $H(m)$ generated by the NSF1 algorithm and let $z^*(m) = \nabla J(\theta(m))$. Then, for all $t \ge 1$, the mean-squared error satisfies
\[
\mathbb{E}\big[\|z(t) - z^*(t)\|_F^2\big]
\leq\mathcal{O}\big(t^{\alpha}\big)
\;+\;
\mathcal{O}\big(t^{1-\alpha}\big)
\;+\;
\mathcal{O}\big(t^{\,1 - 2(\sigma - \alpha)}\big),
\]
where the step-size exponents satisfy $0 <  \alpha < \sigma < 1$.
\end{theorem}

\begin{theorem}[{Finite-Time Convergence to FOSP}]
\label{thm:gradient_bound_mse}
Under Assumptions 1–4, the iterates of the GSF1 algorithm, 
for any horizon $T \ge 1$, satisfy
\begin{align*}
\min_{0 \le m \le T} \mathbb{E}\|\nabla J(\theta(m))\|^2 
&\le \mathcal{O}(t^{\sigma-1}) + \mathcal{O}(t^{-\alpha}) + \mathcal{O}(t^{-2(\sigma - \alpha)})..
\end{align*}

\end{theorem}

Setting $\alpha = 0.4 $ and $\sigma = 0.6$ in Theorem 4 gives us the optimal rate of convergence, i.e., $\mathcal{O}(t^{-0.4}).$

Proofs of both the theorems have been discussed in the appendix.

\section{Experiments and Results}

\subsection{Experimental Setup}

We evaluate the proposed GSF1 and N-SF1 algorithms on a continuous-control benchmark derived from the MountainCar environment in reinforcement learning \cite{c16}. For ease of implementation, we consider the Jacobi version of the N-SF1 algorithm where the Hessian is replaced by a diagonal matrix with diagonal entries corresponding to those of the Hessian and off-diagonal entries set to zero \cite{c6}. To match the theoretical setting of long-run average cost optimization, we implement a \emph{custom infinite-horizon simulator} in which the system dynamics evolve indefinitely without terminal states or episodic resets.

The system state is given by the car's position and velocity, and the control input is a continuous force bounded in $[-1,1]$. The dynamics follow the standard MountainCar equations with a small additive Gaussian noise. The running cost is defined as
\[
h(x,u) = 0.1\,u^2 + (x - x^\star)^2,
\]
where $x^\star = 0.45$ denotes the target position. This cost is smooth, bounded, and satisfies the regularity assumptions required for the finite-time analysis.

The policy is parameterized by a two-layer neural network with $\tanh$ activations. Both algorithms operate in the \emph{average-cost} setting, using per-step costs rather than episodic returns. Performance is evaluated by estimating the steady-state average cost over long rollouts, with evaluations performed periodically during training.

All experiments are conducted using a single long trajectory per run. We have presented the results of our experiments in figure \ref{fig:GSF}, which were obtained by taking average over 5 random seeds.

\subsection{Hyperparameter Configuration}

Table~\ref{tab:hyperparams} summarizes the hyperparameters used for G-SF1 and diagonal N-SF1. Step sizes are chosen to satisfy the required time-scale separation conditions in the finite-time analysis, with the parameter update evolving on the slowest time scale.
\subsection{Summary of Experimental Findings}

The experimental results shown in the figures below confirm the key theoretical insights of this work. Both G-SF1 and N-SF1 converge in the average-cost setting as predicted, with finite-time behavior strongly influenced by estimator variance.

\begin{figure}[!b]
    \centering
    \includegraphics[width=\linewidth/2]{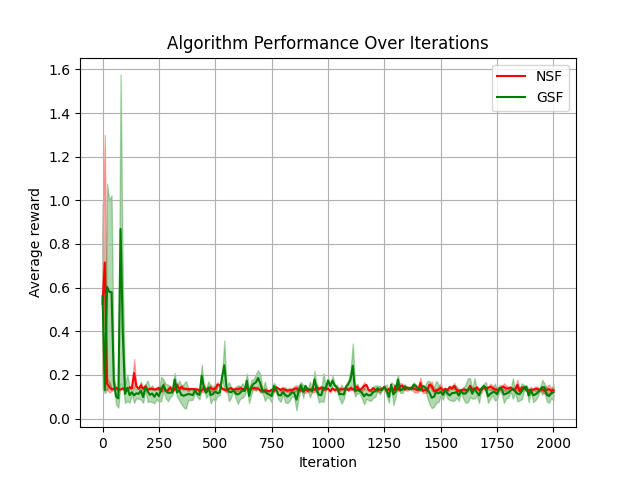}
    \caption{Convergence of algorithms for Mountain Car environment}
    \label{fig:GSF}
\end{figure}

\begin{table}[!t]
\centering
\caption{Hyperparameters used for GSF1 and diagonal NSF1 experiments}
\label{tab:hyperparams}
\begin{tabular}{lcc}
\toprule
Hyperparameter & GSF1 & Diagonal NSF1 \\
\midrule
$\beta$ & 0.05 & 0.05 \\
$L$ & 400 & 400 \\
$a(0)$ & 0.1 & 0.1 \\
$b(0)$ & 1 & 0.5 \\
$c(0)$ & --- & 0.05 \\
$\sigma$ & 0.6 & 0.6 \\
$\alpha$ & 0.4 & 0.45 \\
$\nu$ & --- & 0.40 \\
$\epsilon$ & --- & $10^{-3}$ \\
\bottomrule
\end{tabular}
\end{table}
\section{CONCLUSIONS}
We presented the first finite-time analysis of multi-timescale stochastic optimization algorithms. We considered the GSF1 and NSF1 algorithms in the setting where the performance objective is a long-run average over noisy cost samples when the underlying (noise) process is a parameter-dependent Markov chain.




\clearpage
\onecolumn
\appendix

\setlength{\parskip}{6pt}
\setlength{\parindent}{0pt}
\allowdisplaybreaks
\section{Appendix: Finite-Time Analysis of NSF1 and GSF1}

\subsection*{A.1 Assumptions Used in the Analysis}

We recall the assumptions used throughout the proof:

\begin{itemize}
\item \textbf{(A1)} For all $\theta \in \mathcal{C}$, the Markov chain is irreducible and aperiodic, and the fundamental matrix $Z_\theta$ is uniformly bounded.
\item \textbf{(A2)} The mappings $\theta \mapsto P_\theta$ and $\theta \mapsto h_\theta$ are $C^3$, implying boundedness of derivatives up to order 3.
\item \textbf{(Lemma 1)} $\nabla J(\theta)$ and $\nabla^2 J(\theta)$ are Lipschitz continuous on $\mathcal{C}$.
\item \textbf{(Lemma 2)} $\sup_{\theta \in \mathcal{C}} \mathbb{E}[h(\theta,X)^4] < \infty$.
\item \textbf{(A3)} Step-sizes: $a(m) = (1+m)^{-\sigma}$, $c(m) = (1+m)^{-\alpha}$, $b(m) = (1+m)^{-\nu}$ with $0 < \nu < \alpha < \sigma < 1$.
\item \textbf{(A4)} $M(m) = (H(m)+\lambda I)^{-1}$ is uniformly bounded and positive definite.
\end{itemize}

\subsection*{A.2 Mean-Squared Error Decomposition}

Define
\[
\epsilon_{i,j}(m) = H_{i,j}(m) - H^*_{i,j}(m),
\]
where $H^*(m) = \nabla^2 J(\theta(m))$.

Then,
\[
\|H(m)-H^*(m)\|_F^2
= \sum_{i=1}^d \epsilon_{i,i}(m)^2
+ \sum_{i \neq j} \epsilon_{i,j}(m)^2.
\]

We analyze diagonal terms; off-diagonal terms follow similarly.

\subsection*{A.3 Recursion for Diagonal Error}

From the update,
\[
H_{i,i}(m+1)
= H_{i,i}(m) + b(m)\left[\frac{\eta_i(m)^2-1}{\beta^2}h(X_m) - H_{i,i}(m)\right],
\]

subtracting $H^*_{i,i}(m+1)$ gives:
\begin{align*}
\epsilon_{i,i}(m+1)
&= \epsilon_{i,i}(m) - b(m)\epsilon_{i,i}(m) \\
&\quad + b(m)\left[\frac{\eta_i(m)^2-1}{\beta^2}h(X_m) - H^*_{i,i}(m)\right] \\
&\quad + \left(H^*_{i,i}(m) - H^*_{i,i}(m+1)\right).
\end{align*}

Define
\[
\xi_i(m) := \frac{\eta_i(m)^2-1}{\beta^2}h(X_m) - H^*_{i,i}(m),
\quad
\Delta_i(m) := H^*_{i,i}(m) - H^*_{i,i}(m+1).
\]

Thus,
\[
\epsilon_{i,i}(m+1)
= (1-b(m))\epsilon_{i,i}(m)
+ b(m)\xi_i(m)
+ \Delta_i(m).
\]

\subsection*{A.4 Squaring and Expanding}

Squaring both sides:
\begin{align*}
\epsilon_{i,i}(m+1)^2
&= (1-b(m))^2 \epsilon_{i,i}(m)^2
+ b(m)^2 \xi_i(m)^2
+ \Delta_i(m)^2 \\
&\quad + 2(1-b(m))b(m)\epsilon_{i,i}(m)\xi_i(m) \\
&\quad + 2(1-b(m))\epsilon_{i,i}(m)\Delta_i(m) \\
&\quad + 2b(m)\xi_i(m)\Delta_i(m).
\end{align*}

Using $(1-b(m))^2 \le 1 - 2b(m)$,
\begin{align*}
\epsilon_{i,i}(m+1)^2
&\le (1-2b(m))\epsilon_{i,i}(m)^2
+ 2b(m)\epsilon_{i,i}(m)\xi_i(m) \\
&\quad + 2\epsilon_{i,i}(m)\Delta_i(m)
+ 2b(m)^2 \xi_i(m)^2
+ 2\Delta_i(m)^2.
\end{align*}

Rearranging,
\begin{align*}
\epsilon_{i,i}(m)^2
&\le \frac{1}{2b(m)}\big(\epsilon_{i,i}(m)^2 - \epsilon_{i,i}(m+1)^2\big) \\
&\quad + \epsilon_{i,i}(m)\xi_i(m)
+ \frac{1}{b(m)}\epsilon_{i,i}(m)\Delta_i(m) \\
&\quad + 2b(m)\xi_i(m)^2
+ \frac{1}{b(m)}\Delta_i(m)^2.
\end{align*}

\subsection*{A.5 Summation and Term-wise Bounds}

Let
\[
F(t) := \sum_{m=\tau}^t \mathbb{E}[\epsilon_{i,i}(m)^2].
\]

We now bound each term.

\paragraph{Term C1 (Telescoping)}

Using boundedness of $H(m)$ and $H^*(m)$ (from compactness + smoothness via A1–A2),
\[
\epsilon_{i,i}(m)^2 \le C.
\]

Thus,
\[
\sum \frac{1}{2b(m)}(\epsilon^2(m) - \epsilon^2(m+1))
\le \frac{C}{b(t)} = O(t^\nu).
\]

\paragraph{Term C2 (Noise cross term)}

Using independence of $\eta(m)$ from the past and $\mathbb{E}[\eta_i^2-1]=0$,
\[
\mathbb{E}[\epsilon_{i,i}(m)\xi_i(m)]
\le \sqrt{\mathbb{E}[\epsilon_{i,i}(m)^2]} \sqrt{\mathbb{E}[\xi_i(m)^2]}.
\]

By Lemma 2,
\[
\mathbb{E}[\xi_i(m)^2] \le C.
\]

Hence,
\[
\sum \mathbb{E}[\epsilon_{i,i}(m)\xi_i(m)]
\le C_2 \sqrt{F(t)}.
\]

\paragraph{Term C3 (Drift term)}

By Lipschitz continuity of the Hessian (Lemma 1),
\[
\|\Delta_i(m)\| \le L_H \|\theta(m+1)-\theta(m)\|.
\]

From the update and boundedness of $M(m)$ (A4) and $z(m)$ (Lemma 2),
\[
\|\theta(m+1)-\theta(m)\| \le C a(m).
\]

Thus,
\[
\|\Delta_i(m)\| \le C a(m).
\]

Applying Cauchy–Schwarz,
\[
\sum \frac{\epsilon(m)\Delta(m)}{b(m)}
\le C_3 \sqrt{F(t)} \sqrt{\sum \left(\frac{a(m)}{b(m)}\right)^2}.
\]

\paragraph{Term C4 (Variance term)}

Using $\mathbb{E}[(\eta_i^2-1)^2]=2$ and Lemma 2,
\[
\mathbb{E}[\xi_i(m)^2] \le C.
\]

Thus,
\[
\sum b(m)\mathbb{E}[\xi_i(m)^2]
\le C_4 \sum b(m)
= O(t^{1-\nu}).
\]

\paragraph{Term C5 (Drift squared)}

Using Lipschitz Hessian and parameter increment bound,
\[
\Delta_i(m)^2 \le C a(m)^2.
\]

Hence,
\[
\sum \frac{1}{b(m)}\mathbb{E}[\Delta_i(m)^2]
\le C_5 \sum \frac{a(m)^2}{b(m)}
= O(t^{1-(2\sigma-\nu)}).
\]

\subsection*{A.6 Combining All Terms and Detailed Resolution}

From the bounds obtained for the five terms, we have:
\begin{align}
F(t)
&\le \frac{C_1}{b(t)}
+ C_2 \sqrt{F(t)}
+ C_3 \sqrt{F(t)} \sqrt{\sum_{m=\tau}^t \left(\frac{a(m)}{b(m)}\right)^2} \nonumber \\
&\quad + C_4 \sum_{m=\tau}^t b(m)
+ C_5 \sum_{m=\tau}^t \frac{a(m)^2}{b(m)}.
\label{eq:A6-start}
\end{align}

Using the step-sizes:
\[
a(t) = (1+t)^{-\sigma}, \quad b(t) = (1+t)^{-\nu},
\quad 0 < \nu < \sigma < 1.
\] in \ref{eq:A6-start}, we get, 

\begin{align*}
   F(t)  &\leq C_1(1+t)^\nu + C_2 \sqrt{F(t)} + C_3\sqrt{F(t)}\sqrt{\sum\limits_{t=\tau}^t(1+t)^{2(\nu-\sigma)}} \\
     & + C_4 \sum\limits_{t=\tau}^t(1+t)^{-\nu} + C_5\sum\limits_{t=\tau}^t(1+t)^{(\nu-2\sigma)}.
\end{align*}

Now observe that, 

 \[0 < \nu < \sigma < 1 \implies (1+t)^{2\nu} < (1+t)^{2\sigma}
         \implies(1+t)^\nu < (1+t)^{2\sigma-\nu} \implies (1+t)^{\nu - 2\sigma} < (1+t)^{-v}.\]

Using this, we obtain, 

\begin{align*}
   F(t)  &\leq C_1(1+t)^\nu + C_2 \sqrt{F(t)} + C_3\sqrt{F(t)}\sqrt{\sum\limits_{t=\tau}^t(1+t)^{2(\nu-\sigma)}} \\
     & + (C_4 + C_5)\sum\limits_{t=\tau}^t(1+t)^{-\nu}.
\end{align*}

Observe the term, $\sum\limits_{t=\tau}^t(1+t)^{2(\nu-\sigma)}$
$\sum\limits_{t=\tau}^t(1+t)^{2(\nu-\sigma)} = \sum\limits_{t=\tau}^t\frac{1}{(1+t)^{2(\sigma - \nu)}}$
         \newline
         \newline
         Let $n = 2(\sigma - \nu)$, then $n >0$ as $\sigma >\nu$.
         \newline 
         \newline
         $\sum\limits_{t=\tau}^t\frac{1}{(1+t)^n} \leq \sum\limits_{t=0}^{t-\tau}\frac{1}{(1+t)^n},$
         \newline
         $ \sum\limits_{t=0}^{t-\tau}\frac{1}{(1+t)^n}\leq \int_0^{t-\tau}\frac{1}{(1+t)
         ^n}dt = [\frac{(1+t)^{1-n}}{1-n}]_0^{t-\tau} = [\frac{(1+t-\tau)^{1-n}-1}{1-n}]$
        \newline
        Assuming $1-n > 0 \implies 2(\sigma-\nu) <1\implies(\sigma - \nu)<\frac{1}{2}$, we get,
        \newline
         $\sum\limits_{t=\tau}^{t}\frac{1}{(1+t)^n} \leq [\frac{(1+t-\tau)^{1-n}}{1-n}]$,
         \newline
         Finally we get,
         $\sum\limits_{t=\tau}^{t}(1+t)^{2(\nu-\sigma)} \leq [\frac{(1+t-\tau)^{1-2(\sigma - \nu)}}{1-2(\sigma - \nu)}].$
\\
Substituting this back in \ref{eq:A6-start}, we get,
\begin{align*}
   F(t)  &\leq C_1(1+t)^\nu + C_2 \sqrt{F(t)} + C_3\sqrt{F(t)}\sqrt{\frac{(1+t-\tau)^{1-2(\sigma - \nu)}}{1-2(\sigma - \nu)}} \\
     & + (C_4 + C_5)\sum\limits_{t=\tau}^t(1+t)^{-\nu}.
\end{align*}
\\
Let, \newline $A(t) \coloneq C_1(1+t)^\nu + (C_4 + C_5)(\sum\limits_{m=\tau}^t(1+m)^{-\nu})$,
     \newline $H(t) \coloneq [C_2 + C_3\sqrt{\frac{(1+t-\tau)^{1-2(\sigma-v)}}{1-2(\sigma-\nu)}} ]^2$ ,and $B \coloneq\frac{1}{2}$.
     \newline
     \newline
     Then can then be rewritten as ,
     \newline
     $F(t) \leq A(t) + 2B\sqrt{H(t)}\sqrt{F(t)},$
     \newline
     \newline
     $F(t) - 2B\sqrt{H(t)}\sqrt{F(t)}\leq A(t),$
     \newline
     \newline
     $F(t) - 2B\sqrt{H(t)}\sqrt{F(t)} + B^2 H(t) \leq A(t) + B^2 H(t),$
     \newline
     \newline
     $\sqrt{F(t)} - B\sqrt{H(t)} \leq \sqrt{A(t)} + \sqrt{B^2H(t)},$
     \newline
     \newline
     $\sqrt{F(t)} \leq \sqrt{A(t)} + 2\sqrt{B^2H(t)},$
     \newline
     \newline
     $F(t) \leq 2A(t) + 4B^2H(t).$
     \newline
     \newline
Therefore, 
$F(t) \leq 2C_1(1+t)^\nu  +2(C_4 + C_5)(\frac{(1+t-\tau)^{1-\nu}}{1- \nu})  + [C_2 + C_3\sqrt{\frac{(1+t-\tau)^{1-2(\sigma-v)}}{1-2(\sigma-\nu)}} ]^2$
\newline
Now,
\newline
$F(t) \leq 2C_1(1+t)^\nu  +\frac{2(C_4 + C_5)}{1-\nu}((1+t-\tau)^{1-\nu})  + C_2^2 + \frac{C_3^2}{1-2(\sigma-\nu)}(1+t-\tau)^{1-2(\sigma-v)} + \frac{2C_2C_3}{\sqrt{1-2(\sigma-\nu)}}(1+t-\tau)^{\frac{1-2(\sigma-\nu)}{2}},$
     \newline
     \newline
     $F(t)\leq 2C_1(1+t)^\nu  +\frac{2(C_4 + C_5)}{1-\nu}(1+t-\tau)^{1-\nu} + C_2^2(1+t-\tau)^{1-2(\sigma-v)} + \frac{C_3^2}{1-2(\sigma-\nu)}(1+t-\tau)^{({1-2(\sigma-v)})} + \frac{2C_2C_3}{\sqrt{1-2(\sigma-\nu)}}(1+t-\tau)^{\frac{1-2(\sigma-\nu)}{2}},$
     \newline
     \newline
     $F(t)\leq 2C_1(1+t)^\nu  +\frac{2(C_4 + C_5)}{1-\nu}(1+t-\tau)^{1-\nu} + C_2^2(1+t-\tau)^{1-2(\sigma-v)} + \frac{C_3^2}{1-2(\sigma-\nu)}(1+t-\tau)^{({1-2(\sigma-v)})} + \frac{2C_2C_3}{\sqrt{1-2(\sigma-\nu)}}(1+t-\tau)^{1-2(\sigma-\nu)},$
     \newline
     \newline
     $F(t) \le 2C_1(1+t)^\nu  +\frac{2(C_4 + C_5)}{1-\nu}(1+t-\tau)^{1-\nu} + [C_2^2 + \frac{C_3^2}{1-2(\sigma-\nu)} + \frac{2C_2C_3}{\sqrt{1-2(\sigma-\nu)}}](1+t-\tau)^{1-2(\sigma-\nu)}$.

\paragraph{Final Result}

Therefore,
\[
F(t)
= O(t^\nu) + O(t^{1-\nu}) + O(t^{1-2(\sigma-\nu)}).
\]

which matches the bound stated in Theorem 1.

\subsection*{A.7 Proof of Theorem 2: Gradient Convergence for NSF1}

We now derive a finite-time bound on
\[
G(t) := \sum_{m=\tau}^t \mathbb{E}\|\nabla J(\theta(m))\|^2.
\]

From Lipschitz continuity of the gradient (Lemma 1),
\[
J(\theta(m+1))
\le J(\theta(m))
+ \nabla J(\theta(m))^T (\theta(m+1)-\theta(m))
+ \frac{L_G}{2}\|\theta(m+1)-\theta(m)\|^2.
\]

From the update:
\[
\theta(m+1) = \theta(m) - a(m)M(m)z(m),
\]

we get:
\begin{align*}
J(\theta(m+1))
&\le J(\theta(m))
- a(m)\nabla J(\theta(m))^T M(m)z(m) \\
&\quad + \frac{L_G}{2}a(m)^2 \|M(m)z(m)\|^2.
\end{align*}

We write:
\[
M(m)z(m) = M^*(m)\nabla J(\theta(m))
+ (M(m)-M^*(m))z(m)
+ M^*(m)(z(m)-\nabla J(\theta(m))).
\]

Substituting:
\begin{align*}
\nabla J^T M z
&= \nabla J^T M^* \nabla J
+ \nabla J^T (M-M^*)z \\
&\quad + \nabla J^T M^*(z-\nabla J).
\end{align*}

Since $M^*(m)$ is uniformly positive definite (A4),
\[
\nabla J(\theta(m))^T M^*(m)\nabla J(\theta(m))
\ge \lambda_c \|\nabla J(\theta(m))\|^2.
\]

Substitute into recursion, thus,
\begin{align}
J(\theta(m+1))
&\le J(\theta(m))
- a(m)\lambda_c \|\nabla J(\theta(m))\|^2 \nonumber \\
&\quad + a(m)\|\nabla J(\theta(m))\| \|(M-M^*)z\| \nonumber \\
&\quad + a(m)\|\nabla J(\theta(m))\| \|M^*(z-\nabla J)\| \nonumber \\
&\quad + \frac{L_G}{2}a(m)^2 \|M(m)z(m)\|^2.
\label{eq:grad-main}
\end{align}

Taking expectation and rearranging terms, we get, 

\begin{align*}
    \sum\limits_{m=\tau}^t  E[||\nabla J(\theta(m))||^2] \leq & \sum\limits_{m=\tau}^tE[\frac{1}{a(m)\lambda_c}(J(\theta(m) )-J(\theta(m+1)))] \\
    & + \sum\limits_{m=\tau}^tE[ \frac{1}{\lambda_c} \langle \nabla J(\theta(m)), -[M(m) - M^*(m)]z(m)\rangle] \\
    &+ \sum\limits_{m=\tau}^tE[ \frac{1}{\lambda_c}\langle\nabla J(\theta(m)), -[M^*(m)][z(m) - z^*(m)]\rangle]+ \sum\limits_{m=\tau}^t E[ \frac{La(m)}{2\lambda_c}||M(m)z(m)||^2]
\end{align*}

Now bounding each of the term in the RHS, we get,

\textbf{\underline{Term 1:-}}
      \newline
      $\sum\limits_{m=\tau}^t  E[\frac{J(\theta(m) - J(\theta(m+1)}{a(m)\lambda_c}]$
      \newline
      \newline
      $= \frac{1}{\lambda_c}E[\sum\limits_{m=\tau}^t\frac{J(\theta(m) - J(\theta(m+1)}{a(m)}]$.
      \newline
      \newline
      Now consider, 
      $\sum\limits_{m=\tau}^t\frac{J(\theta(m) - J(\theta(m+1)}{a(m)}$
      \newline
      \newline
      $=\frac{J(\theta(\tau)) - J(\theta(\tau+1))}{a(\tau)} + \frac{J(\theta(\tau +1)) - J(\theta(\tau+2))}{a(\tau+1)} + ... + \frac{J(\theta(t)) - J(\theta(t+1))}{a(t)}$,
      \newline
      \newline
      $=\frac{J(\theta(\tau))}{a(\tau)} + J(\theta(\tau+1))[\frac{1}{a(\tau+1)} - \frac{1}{a(\tau)}] + ... + J(\theta(t))[\frac{1}{a(t)} - \frac{1}{a(t-1)}] -\frac{J(\theta(t+1)}{a(t)}$,
      \newline
      $=\frac{J(\theta(\tau))}{a(\tau)} + \sum\limits_{m=\tau}^tJ(\theta(m))[\frac{1}{a(m)} - \frac{1}{a(m-1}]  -\frac{J(\theta(t+1)}{a(t)}$,
      \newline
      \newline
      $\leq R[\frac{1}{a(\tau)} + \frac{1}{a(\tau + 1)} - \frac{1}{a(\tau)} + \frac{1}{a(\tau + 2)} - \frac{1}{a(\tau + 1)}+...+\frac{1}{a(t)} - \frac{1}{a(t+1)} + \frac{1}{a(t)}]  $
      \newline(using $|J(\theta(m)| \leq R$, from continuity of J and compactness of parameter space.)
      \newline
      \newline
      $=\frac{2R}{a(t)}$.
      \newline
      Therefore, $\sum\limits_{m=\tau}^t\frac{J(\theta(m) - J(\theta(m+1)}{a(m)} \leq \frac{2R}{a(t)}$.
      \newline
      \newline
      Thus,
      $\sum\limits_{m=\tau}^t  E[\frac{J(\theta(m) - J(\theta(m+1)}{a(m)\lambda_c}] \leq \frac{C_1}{a(t)}$,
      \newline where $C_1 = \frac{2R}{\lambda_c}$.
      \\
      \textbf{\underline{Term 2:-}}
      \newline
      $\sum\limits_{m=\tau}^tE[\frac{1}{\lambda_c}\langle\nabla J(\theta(m), -[M(m)-M^*(m)]z(m)\rangle]$
      \newline
      \newline
      $=\sum\limits_{m=\tau}^t\frac{1}{\lambda_c}E[\langle\nabla J(\theta(m), -[M(m)-M^*(m)]z(m)\rangle]$,
      \newline
      \newline
      $\leq\frac{1}{\lambda_c}\sum\limits_{m=\tau}^t\sqrt{E[||\nabla J(\theta(m))||^2]}\sqrt{E[||[M(m) -M^*(m)]z(m)||^2]}$,
      \newline
      \newline
      $\leq\frac{1}{\lambda_c}\sum\limits_{m=\tau}^t\sqrt{E[||\nabla J(\theta(m))||^2]}\sqrt{E[||M(m) -M^*(m)||^2_2||z(m)||^2]}$,
      \newline
      \newline
      $\leq \frac{C_G}{\lambda_c}\sqrt{\sum\limits_{m=\tau}^tE[||\nabla J(\theta(m))||^2]}\sqrt{\sum\limits_{m=\tau}^tE[||M(m)-M^*(m)||^2_2]}$,
      \newline (using that the gradient estimates are bounded)
      \newline
      \newline
       $= C_2\sqrt{\sum\limits_{m=\tau}^tE[||\nabla J(\theta(m))||^2]}\sqrt{\sum\limits_{m=\tau}^tE[||M(m)-M^*(m)||^2_2]}$,
      \newline where $C_2=\frac{C_G}{\lambda_c} $.
      \\
      \textbf{\underline{Text 3:-}}
      \newline
      $\sum\limits_{m=\tau}^tE[\frac{1}{\lambda_c}\langle J(\theta(m)), -[M^*(m)[z(m)-z^*(m)]]\rangle]$
      \newline
      \newline
      $=\frac{1}{\lambda_c}\sum\limits_{m=\tau}^tE[\langle J(\theta(m)), -[M^*(m)[z(m)-z^*(m)]]\rangle]$
      \newline
      \newline
      $\leq\frac{1}{\lambda_c}\sum\limits_{m=\tau}^t\sqrt{E[||\nabla J(\theta(m))||^2]}\sqrt{E[||[M^*(m)][z(m)-z^*(m)]||^2]}$,
      \newline
      \newline
      $\leq\frac{1}{\lambda_c}\sum\limits_{m=\tau}^t\sqrt{E[||\nabla J(\theta(m))||^2]}\sqrt{E[||[M^*(m)]||^2_F||[z(m)-z^*(m)]||^2]}$,
      \newline
      \newline
      $\leq\frac{C_H}{\lambda_c}\sum\limits_{m=\tau}^t\sqrt{E[||\nabla J(\theta(m))||^2]}\sqrt{E[||[z(m)-z^*(m)]||^2]}$,
      \newline
      \newline
      $\leq \frac{C_H}{\lambda_c}\sqrt{\sum\limits_{m=\tau}^tE[||\nabla J(\theta(m))||^2]}\sqrt{\sum\limits_{m=\tau}^tE[||[z(m)-z^*(m)]||^2]}$,
      \newline
      \newline
      $= C_3\sqrt{\sum\limits_{m=\tau}^tE[||\nabla J(\theta(m))||^2]}\sqrt{\sum\limits_{m=\tau}^tE[||[z(m)-z^*(m)]||^2]}$,
      \newline where $ C_3 = \frac{C_H}{\lambda_c}$.
      \\
      \textbf{\underline{Term 4:-}}
      \newline
      $\sum\limits_{m=\tau}^tE[\frac{La(m)}{\lambda_c}||M(m)z(m)||^2]$
      \newline
      \newline
      $=\frac{L}{\lambda_c}\sum\limits_{m=\tau}^tE[a(m)||M(m)z(m)||^2]$,
      \newline
      \newline
      $=\frac{L}{\lambda_c}\sum\limits_{m=\tau}^ta(m)E[||M(m)z(m)||^2]$,
      \newline
      \newline
      $\leq \frac{L}{\lambda_c}\sum\limits a(m) E[||M(m)||_F^2||z(m)||^2]$,
      \newline
      \newline
      $\leq \frac{LC_HC_G}{\lambda_c}\sum\limits_{m=\tau}^ta(m)$,
      \newline
      \newline
      $=C_4\sum\limits_{m=\tau}^ta(m) $.
      \newline where, $C_4 = \frac{LC_HC_G}{\lambda_c}$.
    \newpage
       Combining all the bounds obtained, we get,
     \newline
     $\sum\limits_{m=\tau}^tE[||\nabla J(\theta(m))||^2] \leq \frac{C_1}{a(t)} + C_2 \sqrt{\sum\limits_{m=\tau}^tE[||\nabla J(\theta(m))||^2]}\sqrt{\sum\limits_{m=\tau}^tE[||M(m)-M^*(m)||^2]} + $
     \newline$C_3\sqrt{\sum\limits_{m=\tau}^tE[||\nabla J(\theta(m))||^2]}\sqrt{\sum\limits_{m=\tau}^tE[||z(m)-z^*(m)||^2]} + C_4 \sum\limits_{m=\tau}^t a(m)$.
     \newline
     \newline
     Now substituting $a(t) = \frac{1}{(1+t)^\sigma}$
     \newline
     $\sum\limits_{m=\tau}^tE[||\nabla J(\theta(m))||^2] =C_1(1+t)^\sigma + C_2 \sqrt{\sum\limits_{m=\tau}^tE[||\nabla J(\theta(m))||^2]}\sqrt{\sum\limits_{m=\tau}^tE[||M(m)-M^*(m)||^2]} +$
     \newline$ C_3\sqrt{\sum\limits_{m=\tau}^tE[||\nabla J(\theta(m))||^2]}\sqrt{\sum\limits_{m=\tau}^tE[||z(m)-z^*(m)||^2]} + C_4 \sum\limits_{m=\tau}^t (1+m)^{-\sigma}$
     \newline
     \newline
     Using the bound obtained while proving Theorem 1,
     \newline
     $\sum\limits_{m=\tau}^tE[||\nabla J(\theta(m))||^2] \leq C_1(1+t)^\sigma + C_2 \sqrt{\sum\limits_{m=\tau}^tE[||\nabla J(\theta(m))||^2]}\sqrt{\sum\limits_{m=\tau}^tE[||M(m)-M^*(m)||^2]} +$
     \newline
     $ C_3\sqrt{\sum\limits_{m=\tau}^tE[||\nabla J(\theta(m))||^2]}\sqrt{\sum\limits_{m=\tau}^tE[||z(m)-z^*(m)||^2]} + \frac{C_4}{1-\sigma} (1+t-\tau)^\sigma$   
     \newline
     \newline
     Now dividing both the sides with $1 + t - \tau$, we get, 
     \newline
     $\frac{\sum\limits_{m=\tau}^tE[||\nabla J(\theta(m))||^2]}{1+t-\tau} \leq  \frac{C_1(1+t)^\sigma}{1+t-\tau}+C_2\sqrt{\frac{\sum\limits_{m=\tau}^tE[||\nabla J(\theta(m))||^2]}{1+t-\tau}}\sqrt{\frac{\sum\limits_{m=\tau}^tE[||M(m)-M^*(m)||^2]}{1+t-\tau}}+$
     \newline
     $C_3\sqrt{\frac{\sum\limits_{m=\tau}^tE[||\nabla J(\theta(m))||^2]}{1+t-\tau}}\sqrt{\frac{\sum\limits_{m=\tau}^tE[||z(m)-z^*(m)||^2]}{1+t-\tau}} + \frac{C_4}{(1-\sigma)(1+t-\tau)}(1+t-\tau)^{1-\sigma}$
    \newline
    \newline
    Let $F(t) \coloneq \frac{\sum\limits_{m=\tau}^tE[||\nabla J(\theta(m))||^2]}{1+t-\tau},$
    \newline
    $A(t) \coloneq \frac{C_1(1+t)^\sigma}{1+t-\tau} + \frac{C_4}{(1-\sigma)(1+t-\tau)}(1+t-\tau)^{1-\sigma},$
    \newline
    $Z(t) \coloneq [C_2\sqrt{\frac{\sum\limits_{m=\tau}^tE[||M(m)-M^*(m)||^2]}{1+t-\tau}} + C_3\sqrt{\frac{\sum\limits_{m=\tau}^tE[||z(m)-z^*(m)||^2]}{1+t-\tau}}]^2 ,$ and $B = \frac{1}{2}$. 
    \newline
    The above inequality can then be written as,
    \newline
    $F(t) \leq A(t) + 2B\sqrt{F(t)}\sqrt{Z(t)}$,
    \newline
    \newline
    $(\sqrt{F(t)} - \sqrt{Z(t)})^2 \leq A(t) + Z(t)$,
    \newline
    
    Now using, if $a\leq b+c$ and $a, b,c >0 \implies \sqrt{a} \leq \sqrt{b} + \sqrt{c}$, we get,
    \newline
    $\sqrt{F(t)} \leq \sqrt{A(t)} + 2 \sqrt{Z(t)},$
    \newline
    Using $a\leq (b+c) \implies a^2 \leq 2(b^2 + c^2)$, we get,
    \newline
    \newline
    $F(t) \leq 2A(t) + 8Z(t)$
    \newline
    \newline
    $F(t) \leq \frac{2C_1(1+t)^\sigma}{(1+t-\sigma)} + \frac{2C_4(1+t-\tau)^{1-\sigma}}{(1-\sigma)(1+t-\tau)} + 8[C_2\sqrt{\frac{\sum\limits_{m=\tau}^tE[||M(m)-M^*(m)||^2]}{1+t-\tau}} + C_3\sqrt{\frac{\sum\limits_{m=\tau}^tE[||z(m)-z^*(m)||^2]}{1+t-\tau}}]^2$,
    \newline
    \newline
    $F(t)\leq \frac{2C_1(1+t)^\sigma}{(1+t-\sigma)} + \frac{2C_4(1+t-\tau)^{1-\sigma}}{(1-\sigma)(1+t-\tau)} + \frac{16C_2^2\sum\limits_{m=\tau}^tE[||M(m)-M^*(m)||^2]}{1+t-\tau} + \frac{16C_2^2\sum\limits_{m=\tau}^tE[||z(m)-z^*(m)||^2]}{1+t-\tau}$
    \newline
    \newline
    Now using the bounds on the gradient and hessian updates, 
    \newline
    $F(t)\leq \mathcal{O}(t^{\sigma-1}) + \mathcal{O}(t^{-\sigma}) + \mathcal{O}(t^{\nu-1}) + \mathcal{O}(t^{-\nu}) + \mathcal{O}(t^{-2(\sigma-\nu)}) + \mathcal{O}(t^{\alpha-1}) + \mathcal{O}(t^{-\alpha}) + \mathcal{O}(t^{-2(\sigma-\alpha)})$
    \newline
    Using the condition on the step-sizes, we get,  
    $-\nu > -\alpha > -\sigma$, $\nu - 1 > \alpha -1 > \sigma-1$, and $-2(\sigma-\alpha) > -2(\alpha - \nu)$ therefore,
    \newline
    \newline
    $F(t) \leq \mathcal{O}(t^{\sigma-1}) + \mathcal{O}(t^{-\nu}) + \mathcal{O}(t^{-2(\sigma-\alpha)})$
    \newline
    \newline
    \newline
    Now using the fact that, $\min\limits_{0 \leq m \leq T}a(m) \leq \frac{\sum\limits_{m=0}^Ta(m))}{T+1}$
    \newline
    $\min\limits_{0 \leq t \leq T}E[||\nabla J(\theta(t))||^2] \leq \frac{\sum\limits_{m=0}^TE[||\nabla J(\theta(t))||^2]}{1+T}$
    \newline
    \newline
    And then using the bound obtained above proves Theorem 2,
    \newline
    $$\min\limits_{0 \leq t \leq T}E[||\nabla J(\theta(t))||^2] \leq  \mathcal{O}(T^{\sigma-1}) + \mathcal{O}(T^{-\nu}) + \mathcal{O}(T^{-2(\sigma-\alpha)}).$$
    \newline
    \newline

    \subsection*{A.8 Proof of theorem 3: Gradient estimation error}

    The proof structure of Theorem 3 is quite similar to that of Theorem 1, so we will mention just the major steps involved.
    \\
    Define
    \[
\epsilon(m) = ||z(m) - z^*(m)||,
\]
where $z^*(m) = \nabla J(\theta(m))$.
    \newline
    From the update, 

    \[
z(m+1)
= z(m) + b(m)\left[\frac{\eta_i(m)}{\beta}h(X_m) - z(m)\right],
\]
The recursive relation of error is, 

\begin{align*}
\epsilon(m)^2
&\le \frac{1}{2b(m)}\big(\epsilon(m)^2 - \epsilon(m+1)^2\big) \\
&\quad + \langle\epsilon(m), \xi(m) \rangle
+ \frac{1}{b(m)} \langle \epsilon(m), \Delta(m) \rangle \\
&\quad + b(m)||\xi(m)||^2
+ \frac{1}{b(m)}||\Delta(m)||^2.
\end{align*}

where, $\xi(m) = \frac{\eta_i(m)}{\beta}h(X_m)$ and 
$\Delta(m) = z^*(m) - z^*(m+1)$.

Let $F(t) = \sum\limits_{k=\tau}^tE[\epsilon(m)^2] $.
Now bounding all the terms on the RHS, we get the relation,

\begin{align*}
    F(t) \leq &\frac{C_1}{2a(m)} + C_2 \sqrt{F(t)}\\
        &+C_3\sqrt{F(t)}\sqrt{\sum\limits_{k=\tau}^t\frac{a(k)^2}{b(k)^2}} + C4(\sum\limits_{k=\tau}^tb(k)\big)\\
              &+C_5(\sum\limits_{k=\tau}^t\frac{a(k)^2}{b(k)}\big).
\end{align*}

Using the step-size $a(m) = \frac{1}{(1 + m)^\alpha} $ and 
$b(m) = \frac{1}{(1 + m)^\beta}$, 
\newline where $0<\alpha<\beta<1$ and applying algebra similar to  Theorem 1 we get,

\[F(t) \leq \mathcal{O}(t^\beta) + \mathcal{O}(t^{1-\beta}) + \mathcal{O}(t^{1 - 2(\alpha - \beta)}).\]

\subsection*{A.9  Proof of Theorem 4: Gradient Convergence for GSF1}

The proof structure of this theorem is also quite similar to Theorem 4, so here also we will be mentioning just the main steps involved.
\newline
We now derive a finite-time bound on
\[
G(t) := \sum_{m=\tau}^t \mathbb{E}\|\nabla J(\theta(m))\|^2.
\] for GSF1.

From Lipschitz continuity of the gradient (Lemma 1),
\[
J(\theta(m+1))
\le J(\theta(m))
+ \nabla J(\theta(m))^T (\theta(m+1)-\theta(m))
+ \frac{L_G}{2}\|\theta(m+1)-\theta(m)\|^2.
\]

From the update:
\[
\theta(m+1) = \theta(m) - a(m)z(m),
\]
\newline
Using Lipschitz continuity and update rule we get following relation,
\begin{align*}
    \sum_{m=\tau}^tE[||\nabla J(\theta(k)||^2] \leq &\sum_{m=\tau}^t E[\frac{J(\theta(k) - J(\theta(k+1)}{a(k)}] + \sum_{m=\tau}^t E[\langle\nabla J(\theta(k), z^*(k) - z(k) \rangle]]\\
    &+\sum_{m=\tau}^tE[\frac{L}{2}a(k)||z(k)||^2].
\end{align*}

Bounding each term on the RHS, we get,
\newline

\textbf{\underline{Term 1:-}}
This is similar to the Term 1 of Theorem 2, using the telescopic sum argument, we get the bound.
\[
\sum_{m=\tau}^tE[\frac{J(\theta(k) - J(\theta(k+1)}{a(k)}] \leq \frac{2C_1}{a(t)}.
\]

\textbf{\underline{Term 2:-}}
\newline
\begin{align*}
    \sum_{m=\tau}^t E[\langle\nabla J(\theta(k), z^*(k) - z(k) \rangle]] &\leq \sum_{m=\tau}^t  \sqrt{E||\nabla(J(\theta(k)||^2}*\sqrt{E||z^*(k) - z(k)||^2}\\
    &\leq\sqrt{\sum_{m=\tau}^t E||\nabla(J(\theta(k)||^2} \sqrt{\sum_{m=\tau}^t E||z^*(k) + z(k)||^2}.
\end{align*}

\textbf{\underline{Term 3:-}}
\newline
Using the boundedness of the gradient estimate, we get, 
\begin{align*}
   \sum_{m=\tau}^t E[\frac{L}{2}a(k)||z(k)||^2] &=  \sum_{m=\tau}^t\frac{L}{2}a(k)E[||z(k)||^2]\\
   &\leq C_3 \sum_{m=\tau}^t a(k). 
\end{align*}
\newline
Now combining all the bounds obtained and applying steps similar to Theorem 2,  we get,
 \begin{align*}
     \sum_{m=\tau}^tE[||\nabla J(\theta(k)||^2] \leq  &2C_1(1+t)^\alpha + \sqrt{\sum_{m=\tau}^t E||\nabla(J(\theta(k)||^2} \sqrt{\sum_{m=\tau}^t E||z^*(k) + z(k)||^2}\\
     &C_3\frac{(1 + t - \tau)^{1- \alpha}}{1-\alpha}.
 \end{align*}

 Using completing the sum argument as done in Theorem 2, we get, 
\begin{align*}
    \sum_{m=\tau}^tE[||\nabla J(\theta(k)||^2] \leq  \mathcal{O}(t^{\alpha-1}) + \mathcal{O}(t^{-\beta}) + \mathcal{O}(t^{-2(\alpha - \beta)}).
\end{align*}

\end{document}